\documentclass[11pt]{article}
\usepackage{graphicx}
\usepackage{tabularx}
\usepackage{url}
\usepackage{amsthm}
\usepackage{amsmath}
\usepackage{amsfonts}
\usepackage[labelfont=bf]{caption}

\theoremstyle{plain}
\newtheorem{theorem}{Theorem}
\newtheorem{lemma}{Lemma}

\theoremstyle{definition}

\theoremstyle{remark}

\date{}

\begin{document}

\title{Greedy Algorithm for Inference of Decision Trees from
Decision Rule Systems}

\author{Kerven Durdymyradov and Mikhail Moshkov \\
Computer, Electrical and Mathematical Sciences \& Engineering Division \\ and Computational Bioscience Research Center\\
King Abdullah University of Science and Technology (KAUST) \\
Thuwal 23955-6900, Saudi Arabia\\ \{kerven.durdymyradov,mikhail.moshkov\}@kaust.edu.sa
}

\maketitle

\begin{abstract}
Decision trees and decision rule systems play important roles as classifiers, knowledge representation tools, and algorithms. They are easily interpretable models for data analysis, making them widely used and studied in computer science. Understanding the relationships between these two models is an important task in this field.
There are well-known methods for converting decision trees into systems of decision rules. In this paper, we consider the inverse transformation problem, which is not so simple. Instead of constructing an entire decision tree, our study focuses on a greedy polynomial time algorithm that simulates the operation of a decision tree on a given tuple of attribute values.
\end{abstract}


\section{Introduction\label{S1}}

Decision trees \cite{AbouEishaACHM19,AlsolamiACM20,BreimanFOS84,Moshkov05,MoshkovZ11,RokachM07} and systems of decision rules \cite{BorosHIK97,BorosHIKMM00,ChikalovLLMNSZ13,FurnkranzGL12,MPZ08,MoshkovZ11,Pawlak91,PawlakS07} are widely used as classifiers, knowledge representation tools, and algorithms. They are known for their interpretability in data analysis \cite{CaoSJ20,GilmoreEH21,Molnar22,SilvaGKJS20}.

Investigating the relationship between these two models is an important task in computer science. Converting decision trees into decision rule systems is a well-known and simple process \cite{Quinlan87,Quinlan93,Quinlan99}. This paper focuses on the inverse transformation problem, which is not trivial.

The research related to this problem encompasses several  directions:

\begin{itemize}
\item Two-stage construction of decision trees. This approach involves building decision rules based on input data, followed by the construction of decision trees or decision structures (which are generalizations of decision trees) using the generated rules. The benefits of this two-stage construction method are explained in \cite%
{AbdelhalimTN16,AbdelhalimTS09,ImamM93a,ImamM93,ImamM96,KaufmanMPW06,MichalskiI94,MichalskiI97,SzydloSM05}.

\item Relationships between the depth of deterministic and nondeterministic decision trees for computing Boolean functions \cite{BlumI87,HartmanisH87,Moshkov95,Tardos89}. Note that the nondeterministic decision trees can be interpreted as decision rule systems. Note also that the minimum depth of a nondeterministic decision tree for a Boolean function is equal to its certificate complexity \cite{BuhrmanW02}.

\item Relationships between the depth of deterministic and nondeterministic decision trees for problems over finite and infinite information systems \cite{Moshkov96,Moshkov00,Moshkov03,Moshkov05a,Moshkov20}. These systems consist of a universe and a set of attributes defined on it
\cite{Pawlak91}.
\end{itemize}

This paper builds upon the syntactic approach proposed in previous works \cite{Moshkov98,Moshkov01}. The approach assumes that we have a system of decision rules but lack knowledge of the input data, and our objective is to transform these rules into a decision tree.

Let us consider a system of decision rules $S$, which consists of rules of the form
\begin{equation*}
(a_{i_{1}}=\delta_{1})\wedge \cdots \wedge (a_{i_{m}}=\delta_{m})\rightarrow \sigma ,
\end{equation*}
where $a_{i_{1}},\ldots ,a_{i_{m}}$ represent attributes, $\delta _{1},\ldots
,\delta _{m}$ are the corresponding attribute values, and $\sigma$ is the decision.

The problem associated with this system is to determine, for a given input (a tuple of values of attributes included in $S$), all the realizable rules, i.e., rules with a true left-hand side. It is important to note that any attribute in the input can take any value.

The objective of this paper is to minimize the number of queries required to determine the attribute values. To achieve this, decision trees are explored as algorithms for solving the problem at hand.

In our previous work \cite{Kerven23}, we investigated the minimum depth of decision trees for the considered problem and established both upper and lower bounds. These bounds depend on three parameters of the decision rule system: the total number of distinct attributes, the maximum length of a decision rule, and the maximum number of attribute values. We demonstrated that there exist systems of decision rules where the minimum depth of the decision trees is significantly smaller than the total number of attributes in the rule system. This finding shows that decision trees are a reasonable choice for such systems of decision rules.

In another study \cite{Kerven23a}, we examined the complexity of constructing decision trees and acyclic decision graphs that represent decision trees.  We found that in many cases, the minimum number of nodes in decision trees can grow as a superpolynomial function depending on the size of the decision rule systems. To address this issue, we introduced two types of acyclic decision graphs as representations of decision trees. However, simultaneously minimizing the depth and the number of nodes in these graphs poses a challenging bi-criteria optimization problem.

We left this problem for future research and pursued an alternative approach: instead of constructing the entire decision tree, we developed a polynomial time algorithm that models the behavior of the decision tree for a given tuple of attribute values. This algorithm is based on an auxiliary algorithm for the construction of a node cover for a hypergraph corresponding to a decision rule system: nodes of this hypergraph correspond to attributes and edges -- to rules from the rule system. The auxiliary algorithm is not greedy: at each step, this algorithm adds to the cover being constructed all the attributes belonging to a rule that has not yet been covered.

In this paper, we develop a new algorithm with polynomial time complexity, which models the behavior of a decision tree solving the considered problem for a given tuple of attribute values. The auxiliary algorithm for it is a standard greedy algorithm for the set cover problem. Therefore we talk about the entire algorithm for describing the operation of decision trees as greedy. We study the accuracy of this algorithm, i.e., we compare the depth of the decision tree described by it and the minimum depth of a decision tree. The obtained bound is a bit worse in the comparison with the bound  for the algorithm considered in \cite{Kerven23a}. However, we expect that the considered algorithms are mutually complementary: the old one will work better for systems with short decision rules and the new one will work better for systems with long decision rules. In the future, we are planning to do computer experiments to explore this hypothesis.

In this paper, we repeat the main definitions from \cite{Kerven23} and give
some lemmas from \cite{Kerven23} without proofs.

This paper consists of five sections. Section \ref{S2} considers the main
definitions and notation. Section \ref{S3} contains auxiliary statements.
Section \ref{S4} discusses the greedy algorithm, which models the behavior of a decision tree. Section \ref{S5} contains short conclusions.

\section{Main Definitions and Notation\label{S2}}

In this section, we consider the main definitions and notation related to
decision rule systems and decision trees. In fact, we repeat corresponding
definitions and notation from \cite{Kerven23}.

\subsection{Decision Rule Systems\label{S2.1}}

Let $\omega =\{0,1,2,\ldots \}$ and $A=\{a_{i}:i\in \omega \}$. Elements of
the set $A$ will be called \emph{attributes}.

A \emph{decision rule} is an expression of the form
\begin{equation*}
(a_{i_{1}}=\delta _{1})\wedge \cdots \wedge (a_{i_{m}}=\delta
_{m})\rightarrow \sigma ,
\end{equation*}%
where $m\in \omega $, $a_{i_{1}},\ldots ,a_{i_{m}}$ are pairwise different
attributes from $A$ and $\delta _{1},\ldots ,\delta _{m},$ $\sigma \in \omega $.

We denote this decision rule by $r$. The expression $(a_{i_{1}}=\delta
_{1})\wedge \cdots \wedge (a_{i_{m}}=\delta _{m})$ will be called the \emph{%
left-hand side}, and the number $\sigma $ will be called the \emph{%
right-hand side} of the rule $r$. The number $m$ will be called the \emph{%
length }of the decision rule $r$. Denote $A(r)=\{a_{i_{1}},\ldots
,a_{i_{m}}\}$ and $K(r)=\{a_{i_{1}}=\delta _{1},\ldots ,$ $a_{i_{m}}=\delta
_{m}\}$. If $m=0$, then $A(r)=K(r)=\emptyset $.

A \emph{system of decision rules} $S$ is a finite nonempty set of decision
rules. Denote $A(S)=\bigcup_{r\in S}A(r)$, $n(S)=\left\vert A(S)\right\vert $%
,  and $%
d(S)$ the maximum length of a decision rule from $S$. Let $n(S)>0$. For $%
a_{i}\in A(S)$, let $V_{S}(a_{i})=\{\delta :a_{i}=\delta \in \bigcup_{r\in
S}K(r)\}$ and $EV_{S}(a_{i})=V_{S}(a_{i})\cup \{\ast \}$, where the symbol $%
\ast $ is interpreted as a number from $\omega $ that does not belong to the
set $V_{S}(a_{i})$. Letter $E$ here and later means \emph{extended}: we
consider not only values of attributes occurring in $S$ but arbitrary values
from $\omega $. Denote $k(S)=\max \{\left\vert V_{S}(a_{i})\right\vert
:a_{i}\in A(S)\}$. If $n(S)=0$, then $k(S)=0$. We denote by $\Sigma $ the
set of systems of decision rules.

Let $S\in \Sigma $, $n(S)>0$, and $A(S)=\{a_{j_{1}},\ldots ,a_{j_{n}}\}$,
where $j_{1}<\cdots <j_{n}$. Denote  $EV(S)=EV_{S}(a_{j_{1}})\times \cdots
\times EV_{S}(a_{j_{n}})$. For $\bar{\delta}=(\delta _{1},\ldots ,\delta
_{n})\in EV(S)$, denote $K(S,\bar{\delta})=\{a_{j_{1}}=\delta _{1},\ldots
,a_{j_{n}}=\delta _{n}\}$. We will say that a decision rule $r$ from $S$ is
\emph{realizable} for a tuple $\bar{\delta}\in EV(S)$ if $K(r)\subseteq K(S,%
\bar{\delta})$. It is clear that any rule with the empty left-hand side is
realizable for the tuple $\bar{\delta}$.

We now define a problem related to the rule system $S$.

Problem \emph{Extended All Rules}: for a given tuple $\bar{\delta}\in EV(S)$%
, it is required to find the set of rules from $S$ that are realizable for
the tuple $\bar{\delta}$. We denote this problem $EAR(S)$. In the special
case, when $n(S)=0$, all rules from $S$ have the empty left-hand side. In
this case, it is natural to consider the set $S$ as the solution to the
problem $EAR(S)$.

\subsection{Decision Trees\label{S2.2}}

A \emph{finite directed tree with root} is a finite directed tree in which
only one node has no entering edges. This node is called the \emph{root}.
The nodes without leaving edges are called \emph{terminal} nodes. The nodes
that are not terminal will be called \emph{working} nodes. A \emph{complete
path} in a finite directed tree with root is a sequence $\xi
=v_{1},d_{1},\ldots ,v_{m},d_{m},v_{m+1}$ of nodes and edges of this tree in
which $v_{1}$ is the root, $v_{m+1}$ is a terminal node and, for $i=1,\ldots
,m$, the edge $d_{i}$ leaves the node $v_{i}$ and enters the node $v_{i+1}$.

An \emph{extended decision tree over a decision rule system} $S$ is a
labeled finite directed tree with root $\Gamma $ satisfying the following
conditions:

\begin{itemize}
\item Each working node of the tree $\Gamma $ is labeled with an attribute
from the set $A(S)$.

\item Let a working node $v$ of the tree $\Gamma $ be labeled with an
attribute $a_{i}$. Then exactly $\left\vert EV_{S}(a_{i})\right\vert $ edges
leave the node $v$ and these edges are labeled with pairwise different
elements from the set $EV_{S}(a_{i})$.

\item Each terminal node of the tree $\Gamma $ is labeled with a subset of
the set $S$.
\end{itemize}

Let $\Gamma $ be an extended decision tree over the decision rule system $S$%
. We denote by $CP(\Gamma )$ the set of complete paths in the tree $\Gamma $%
. Let $\xi =v_{1},d_{1},\ldots ,v_{m},d_{m},v_{m+1}$ be a complete path in $%
\Gamma $. We correspond to this path a set of attributes $A(\xi )$ and an
equation system $K(\xi )$. If $m=0$ and $\xi =v_{1}$, then $A(\xi
)=\emptyset $ and $K(\xi )=\emptyset $. Let $m>0$ and, for $j=1,\ldots ,m$,
the node $v_{j}$ be labeled with the attribute $a_{i_{j}}$ and the edge $%
d_{j}$ be labeled with the element $\delta _{j}\in \omega \cup \{\ast \}$.
Then $A(\xi )=\{a_{i_{1}},\ldots ,a_{i_{m}}\}$ and $K(\xi
)=\{a_{i_{1}}=\delta _{1},\ldots ,a_{i_{m}}=\delta _{m}\}$. We denote by $%
\tau (\xi )$ the set of decision rules attached to the node $v_{m+1}$.

A system of equations $\{a_{i_{1}}=\delta _{1},\ldots ,a_{i_{m}}=\delta
_{m}\}$, where $a_{i_{1}},\ldots ,a_{i_{m}}\in A$ and $\delta _{1},\ldots
,\delta _{m}\in \omega \cup \{\ast \}$, will be called \emph{inconsistent}
if there exist $l,k\in \{1,\ldots ,m\}$ such that $l\neq k$, $i_{l}=i_{k}$,
and $\delta _{l}\neq \delta _{k}$. If the system of equations is not
inconsistent, then it will be called \emph{consistent}.

Let $S$ be a decision rule system and $\Gamma $ be an extended decision tree
over $S$.

We will say that $\Gamma $ \emph{solves} the problem  $EAR(S)$ if any path $%
\xi \in CP(\Gamma )$ with consistent system of equations $K(\xi )$ satisfies
the following conditions:

\begin{itemize}
\item For any decision rule $r\in \tau (\xi )$, the relation $K(r)\subseteq
K(\xi )$ holds.

\item For any decision rule $r\in S\setminus \tau (\xi )$, the system of
equations $K(r)\cup K(\xi )$ is inconsistent.
\end{itemize}

For any complete path $\xi \in CP(\Gamma )$, we denote by $h(\xi )$ the
number of working nodes in $\xi $. The value $h(\Gamma )=\max \{h(\xi ):\xi
\in CP(\Gamma )\}$ is called the \emph{depth} of the decision tree $\Gamma $.

Let $S$ be a decision rule system. We denote by $h_{EAR}(S)$ the minimum
depth of a decision tree over $S$, which solves the problem $EAR(S)$.

If $n(S)=0$, then there is only one decision tree solving the problem $EAR(S)
$. This tree consists of one node labeled with the set of rules $S$.
Therefore if $n(S)=0$, then $h_{EAR}(S)=0$.

\section{Auxiliary Statements\label{S3}}

In this section, we first give some statements from \cite{Kerven23} and
then we prove a new one.

Let $S$ be a decision rule system and $\alpha =\{a_{i_{1}}=\delta
_{1},\ldots ,a_{i_{m}}=\delta _{m}\}$ be a consistent equation system such
that $a_{i_{1}},\ldots ,a_{i_{m}}\in A$ and $\delta _{1},\ldots ,\delta
_{m}\in \omega \cup \{\ast \}$. We now define a decision rule system $%
S_{\alpha }$. Let $r$ be a decision rule for which the equation system $%
K(r)\cup \alpha $ is consistent. We denote by $r_{\alpha }$ the decision
rule obtained from $r$ by the removal from the left-hand side of $r$ all
equations that belong to $\alpha $. Then $S_{\alpha }$ is the set of
decision rules $r_{\alpha }$ such that $r\in S$ and the equation system $%
K(r)\cup \alpha $ is consistent.

\begin{lemma}
\label{L1} (follows from Lemma 6 \cite{Kerven23}) Let $S$ be a decision rule
system with $n(S)>0$,  $\alpha =\{a_{i_{1}}=\delta _{1},\ldots
,a_{i_{m}}=\delta _{m}\}$ be a consistent equation system such that $%
a_{i_{1}},\ldots ,a_{i_{m}}\in A(S)$ and, for $j=1,\ldots ,m$, $\delta
_{j}\in EV_{S}(a_{i_{j}})$. Then $h_{EAR}(S)\geq h_{EAR}(S_{\alpha })$.
\end{lemma}

We correspond to a decision rule system $S$ a hypergraph $G(S)$ with the set
of nodes $A(S)$ and the set of edges $\{A(r):r\in S\}$. A \emph{node cover}
of the hypergraph $G(S)$ is a subset $B$ of the set of nodes $A(S)$ such
that $A(r)\cap B\neq \emptyset $ for any rule $r\in S$ such that $A(r)\neq
\emptyset $. If $A(S)=\emptyset $, then the empty set is the only node cover
of the hypergraph $G(S)$. Denote by $\beta (S)$ the minimum cardinality of a
node cover of the hypergraph $G(S)$.

\begin{lemma}
\label{L2} (follows from Lemma 7 \cite{Kerven23}) Let $S$ be a decision rule
system. Then $h_{EAR}(S)\geq \beta (S)$.
\end{lemma}

\begin{lemma}
\label{L3} (follows from Lemma 8 \cite{Kerven23}) Let $S$ be a decision rule
system. Then $h_{EAR}(S)\geq d(S)$.
\end{lemma}

Let $S$ be a decision rule system and $S^{\prime }$ be the set of rules of
the length $d(S)$ from $S$. Two decision rules $r_{1}$ and $r_{2}$ from $%
S^{\prime }$ are called \emph{equivalent} if $K(r_{1})=K(r_{2})$. This
equivalence relation provides a partition of the set $S^{\prime }$ into equivalence classes.
We denote by $S^{\max }$ the set of rules that contains exactly one
representative from each equivalence class and does not contain any other
rules.

\begin{lemma}
\label{L4}  Let $S$ be a decision rule system with $n(S)>0$. Then $$h_{EAR}(S)\geq \ln
|S^{\max }|/\ln (k(S)+1).$$
\end{lemma}

\begin{proof} Let $r\in S^{\max }$ and the rule $r$ is equal to $%
(a_{i_{1}}=\delta _{1})\wedge \cdots \wedge (a_{i_{m}}=\delta
_{m})\rightarrow \sigma $. We now define a tuple $\bar{\delta}(r)\in EV(S)$.
For $j=1,\ldots ,m$, the tuple $\bar{\delta}(r)$ in the position
corresponding to the attribute $a_{i_{j}}$ contains the number $\delta _{j}$%
. All other positions of the tuple $\bar{\delta}(r)$ are filled with the
symbol $\ast $. We denote by $S^{\max }(\bar{\delta}(r))$ the set of rules
from $S^{\max }$ that are realizable for the tuple $\bar{\delta}(r)$. One
can show that $S^{\max }(\bar{\delta}(r))=\{r\}$. From here it
follows that the problem $EAR(S)$ has at least $|S^{\max }|$ pairwise
different solutions.

Let $\Gamma $ be a decision tree, which solves the problem $EAR(S)$ and for
which $h(\Gamma )=h_{EAR}(S)$. Then the number of terminal nodes in $\Gamma $
should be at least $|S^{\max }|$. It is easy to show that the number of
terminal nodes in $\Gamma $ is at most $(k(S)+1)^{h(\Gamma )}$. Therefore $%
(k(S)+1)^{h(\Gamma )}\geq |S^{\max }|$ and $h(\Gamma )\geq \ln |S^{\max
}|/\ln (k(S)+1)$. Thus, $h_{EAR}(S)\geq \ln |S^{\max }|/\ln (k(S)+1)$.
\end{proof}

\section{Algorithms\label{S4}}

In this section, we consider an auxiliary algorithm $\mathcal{A}_{greedy}$ that constructs a node cover for the hypergraph corresponding to a decision rule system and an algorithm $\mathcal{A}^{EAR}$ that describes the work of a
decision tree solving the problem $EAR(S)$ for a decision rule system $S$
with $n(S)>0$.

\subsection{Auxiliary Algorithm $\mathcal{A}_{greedy}$\label{S4.1}}

Let $S$ be a decision rule system with $n(S)>0$. First, we describe a
polynomial time algorithm $\mathcal{A}_{greedy}$ for the construction of a
node cover $B$ for the hypergraph $G(S^{\max })$ such that $\left\vert
B\right\vert \leq \beta (S^{\max })\ln |S^{\max }|+1$.\medskip

\medskip \noindent \emph{Algorithm} $\mathcal{A}_{greedy}$ \medskip

\noindent During each step, this algorithm chooses an attribute $a_{i}\in
A(S^{\max })$ with the minimum index $i$, which covers the maximum number of
rules from $S^{\max }$ uncovered during previous steps and add it to the set
$B$ (an attribute $a_{i}$ covers a rule $r\in S^{\max }$ if $a_{i}\in A(r)$%
). The algorithm will finish the work when all rules from $S^{\max }$ are
covered.\medskip

The considered algorithm is essentially a well-known greedy algorithm for the set
cover problem -- see Sect. 4.1.1 of the book \cite{MoshkovZ11}.

\begin{lemma}
\label{L5} (follows from Theorem 4.1 \cite{MoshkovZ11}) Let $S$ be a decision rule
system with $n(S)>0$. The algorithm $\mathcal{A}_{greedy}$  constructs a
node cover $B$ for the hypergraph $G(S^{\max })$ such that $\left\vert
B\right\vert \leq \beta (S^{\max })\ln |S^{\max }|+1$.\medskip
\end{lemma}

\subsection{Greedy Algorithm $\mathcal{A}^{EAR}$\label{S4.2}}

Let $S$ be a decision rule system with $n(S)>0$. We now describe a
polynomial time algorithm $\mathcal{A}^{EAR}$ that, for a given tuple of
attribute values from the set $EV(S)$, describes the work on this tuple of a
decision tree $\Gamma $, which solves the problem $AER(S)$. Note that this algorithm is similar to the algorithm considered in \cite{Kerven23a}.\medskip \pagebreak

\medskip \noindent \emph{Algorithm} $\mathcal{A}^{EAR}$ \medskip

\noindent The work of the decision tree $\Gamma$ consists of rounds.

\emph{First round}. Using the algorithm $\mathcal{A}_{greedy}$, we construct
a node cover $B_{1}$ of the hypergraph $G(S^{\max })$ with $\left\vert
B_{1}\right\vert \leq \beta (S^{\max })\ln |S^{\max }|+1$. The decision tree
$\Gamma $ sequentially computes values of the attributes from $B_{1}$. As a
result, we obtain a system $\alpha _{1}$ consisting of $\left\vert
B_{1}\right\vert $ equations of the form $a_{i_{j}}=\delta _{j}$, where $%
a_{i_{j}}\in B_{1}$ and $\delta _{j}$ is the computed value of the attribute
$a_{i_{j}}$. If $S_{\alpha _{1}}=\emptyset $ or all rules from $S_{\alpha
_{1}}$ have the empty left-hand side, then the tree $\Gamma $ finishes its
work. The result of this work is the set of decision rules $r$ from $S$ for
which the system of equations $K(r)\cup \alpha _{1}$ is consistent. These rules correspond to rules from $S_{\alpha_{1}}$ with the empty left-hand side.
Otherwise, we move on to the second round of the decision tree $\Gamma $
work.

\emph{Second round}. Using the algorithm $\mathcal{A}_{greedy}$, we
construct a node cover $B_{2}$ of the hypergraph $G((S_{\alpha _{1}})^{\max
})$ with $\left\vert B_{2}\right\vert \leq \beta ((S_{\alpha _{1}})^{\max
})\ln |(S_{\alpha _{1}})^{\max }|+1$. The decision tree $\Gamma $
sequentially computes values of the attributes from $B_{2}$. As a result, we
obtain a system $\alpha _{2}$ consisting of $\left\vert B_{2}\right\vert $
equations. If $S_{\alpha _{1}\cup \alpha _{2}}=\emptyset $ or all rules from
$S_{\alpha _{1}\cup \alpha _{2}}$ have the empty left-hand side, then the
tree $\Gamma $ finishes its work. The result of this work is the set of
decision rules $r$ from $S$ for which the system of equations $K(r)\cup
\alpha _{1}\cup \alpha _{2}$ is consistent. These rules correspond to rules from $S_{\alpha _{1}\cup \alpha _{2}}$ with the empty left-hand side. Otherwise, we move on to the
third round of the decision tree $\Gamma $ work, etc., until we obtain empty
system of rules or system in which all rules have empty left-hand side.
\medskip

\begin{theorem}
\label{T1} Let $S$ be a decision rule system with $n(S)>0$. The algorithm $%
\mathcal{A}^{EAR}$ describes the work of a decision tree $\Gamma $, which
solves the problem $EAR(S)$ and for which $h(\Gamma )\leq h_{EAR}(S)^{3}\ln
(k(S)+1)+h_{EAR}(S)$.\medskip
\end{theorem}

\begin{proof}  It is clear that $d(S)>d(S_{\alpha _{1}})>d(S_{\alpha
_{1}\cup \alpha _{2}})>\cdots $. Therefore the number of rounds is at most $%
d(S)$. By Lemma \ref{L3}, $d(S)\leq h_{EAR}(S)$.
We now show that the number
of attributes values of which are computed by $\Gamma $ during each round is
at most $%
h_{EAR}(S)^{2}\ln (k(S)+1)+1$. We consider only the second round: the proofs for other
rounds are similar. From Lemma \ref{L5} it follows that the number of
attributes values of which are computed by $\Gamma $ during the second round
is at most $\beta ((S_{\alpha _{1}})^{\max })\ln |(S_{\alpha _{1}})^{\max
}|+1$. Evidently, $\beta ((S_{\alpha _{1}})^{\max })\leq \beta (S_{\alpha
_{1}})$. By Lemmas \ref{L1} and \ref{L2},  $\beta (S_{\alpha _{1}})\leq
h_{EAR}(S_{\alpha _{1}})\leq h_{EAR}(S)$. Therefore  $\beta ((S_{\alpha
_{1}})^{\max })\leq h_{EAR}(S)$. By Lemmas \ref{L1} and \ref{L4}, $\ln
|(S_{\alpha _{1}})^{\max }|\leq h_{EAR}(S_{\alpha _{1}})\ln (k(S)+1)\leq
h_{EAR}(S)\ln (k(S)+1)$. Therefore the number of attributes values of which
are computed by $\Gamma $ during the second round is at most $%
h_{EAR}(S)^{2}\ln (k(S)+1)+1$. This bound is true for each round.
The number
of rounds is at most $h_{EAR}(S)$. Thus, the depth of the decision tree $%
\Gamma $ is at most $h_{EAR}(S)^{3}\ln (k(S)+1)+h_{EAR}(S)$. 
\end{proof}

\section{Conclusions\label{S5}}

In this paper, we considered a new  algorithm with polynomial time complexity, which models the behavior of a decision tree solving the problem $EAR(S)$ for a given tuple of attribute values. We studied the accuracy of this algorithm:  we compared the depth of the decision tree described by it and the minimum depth of a decision tree. The obtained bound is a bit worse in the comparison with the bound  for the algorithm considered in \cite{Kerven23a}. However, we expect that these two algorithms are mutually complementary: the old one will work better for systems with short decision rules and the new one will work better for systems with long decision rules. In the future, we are planning to do computer experiments to explore this hypothesis. We are also planning to develop a dynamic programming algorithm for the minimization of the depth of decision trees and to compare experimentally the depth of decision trees constructed by the two considered algorithms with the minimum depth.

\subsection*{Acknowledgements}

Research reported in this publication was supported by King Abdullah
University of Science and Technology (KAUST).

\bibliographystyle{spmpsci}
\bibliography{1C}

\end{document}